\documentclass{article}
\usepackage[final]{ProbNum25}

\usepackage[round]{natbib}
\usepackage{bm}

\usepackage{color}
\usepackage{amsmath,amssymb,amsthm,bbold}
\usepackage{enumitem}
\usepackage{subcaption}
\usepackage[capitalise,nameinlink]{cleveref}

\newcommand{\reals}{\mathbb{R}}

\newcommand{\infoop}{\mathcal{I}}

\newcommand{\norm}[1]{\| #1 \|}

\newcommand{\set}[1]{\{ #1 \}}

\newcommand{\perpdirections}{\bar{V}}
\newcommand{\perpscalings}{\bar{\Psi}}

\newcommand{\spatialvar}{z}
\newtheorem{theorem}{Theorem}[section]
\newtheorem{definition}[theorem]{Definition}
\newtheorem{corollary}[theorem]{Corollary}
\newtheorem{remark}[theorem]{Remark}
\probnumtitle{Randomised Postiterations for Calibrated BayesCG}

\probnumauthors{%
\name{Niall Vyas}%
\affiliation{University of Southampton, United Kingdom}%
\and%
\name{Disha Hegde}%
\affiliation{University of Southampton, United Kingdom}%
\and%
\name{Jon Cockayne}%
\affiliation{University of Southampton, United Kingdom}%
}

\probnumabstract{
    The Bayesian conjugate gradient method offers probabilistic solutions to linear systems but suffers from poor calibration, limiting its utility in uncertainty quantification tasks. Recent approaches leveraging postiterations to construct priors have improved computational properties but failed to correct calibration issues. In this work, we propose a novel randomised postiteration strategy that enhances the calibration of the BayesCG posterior while preserving its favourable convergence characteristics. We present theoretical guarantees for the improved calibration, supported by results on the distribution of posterior errors. Numerical experiments demonstrate the efficacy of the method in both synthetic and inverse problem settings, showing enhanced uncertainty quantification and better propagation of uncertainties through computational pipelines.
}

\begin{document}

\section{Introduction} \label{sec:intro}

Probabilistic numerical methods provide a framework for solving numerical problems while quantifying uncertainty arising from finite computational resources.
The Bayesian conjugate gradient method (BayesCG; see \cite{Cockayne2019BayesCG}) is increasingly widely used as such a solver for linear systems of the form $Ax = b$, where the matrix $A$ and vector $b$ are known and the vector $x$ is to be determined.
The method constructs a Gaussian posterior over the solution space, reflecting uncertainty due to limited iterations.
However, several major challenges impact the practicality of BayesCG: (i) there are compelling theoretical and numerical reasons to use a particular prior that results in an intractable posterior, and (ii) the posteriors produced are typically poorly calibrated.
The \emph{postiteration} method of \cite{Reid2023UQ} addresses the first of these issues but not the second.
In this work we propose a randomised version of this method for which we can provide calibration guarantees.

\subsection{Challenges for BayesCG} \label{sec:challenges}

In this section we describe two central challenges that this work will address.

\paragraph{Poor Calibration}
BayesCG is well-known to suffer from poor calibration as recently demonstrated in \citet{Hegde2025GaussSeidel}, due to nonlinearities in its conditioning procedure.
Specifically, BayesCG operates by conditioning a Gaussian prior on information of the form $S(b)^\top A x = S(b)^\top b$, where the search directions $S(b)$ are generated by the Lanczos process \cite[Section 10.1]{Golub2013MatrixComputations}.
The information is therefore nonlinear in $x$, while the posterior is constructed by assuming that $S$ is independent of $x$ to apply \emph{linear} conditioning, for a conjugate Gaussian posterior.

\paragraph{Inverse Prior}
There is a specific choice of prior for BayesCG that has favourable mathematical properties, but is intractable: the \emph{inverse prior} $x \sim \mathcal{N}(x_0, A^{-1})$.
Under this choice the BayesCG posterior mean is identical to that from CG, which can be computed without computing $A^{-1}$, and therefore converges at a rapid geometric rate in the number of iterations performed.
However the posterior covariance still requires explicit computation of $A^{-1}$.
In certain circumstances this requirement has been sidestepped using elegant marginalisation techniques \citep{Wenger2022CAGP}, but for more general applications this remains a challenge.

\subsection{Postiterations}
To address the challenges discussed in \cref{sec:challenges}, recent work has proposed \emph{empirical Bayesian} approaches that construct an $x$-dependent prior \citep{Reid2020BayesCG,Reid2023UQ} in a way that is consistent with the inverse prior to retain the favourable mathematical properties of CG.
This is typically constructed by performing several \emph{postiterations} of CG, and using those iterations to construct the prior.
The major downside of this work is that the posterior remains uncalibrated, which can have serious consequences in applications in which uncertainty is propagated through a computational pipeline, such as in an inverse problem, as we consider in this work.
We therefore propose a straightforward modification of the postiteration procedure under which the posterior has more favourable calibration properties.

\subsection{Contributions}

This paper makes the following contributions to the literature.
\begin{enumerate}
    \item We propose a new algorithm that exploits \emph{randomised} postiterations to achieve a well-calibrated posterior while maintaining the favourable computational properties of \cite{Reid2020BayesCG,Reid2023UQ}. (\cref{sec:methodology})
    \item We prove several theoretical results related to the calibration of this procedure. (\cref{thm:perturbation} and \cref{corr:calibration})
    \item We test the calibration properties of the posterior, demonstrating that it performs favourably compared to existing approaches, and also show that it has favourable performance in uncertainty-propagation applications. (\cref{sec:experiments})
\end{enumerate}

\subsection{Structure of the Paper}

The rest of the paper proceeds as follows.
In \cref{sec:background} we present the required background on BayesCG, calibration and postiteration methods.
\cref{sec:methodology} introduces our randomised postiteration method.
\cref{sec:experiments} presents numerical results related to calibration of the novel method and how it can be propagated through numerical pipelines.
We conclude with some discussion in \cref{sec:conclusion}. 

\section{Background} \label{sec:background}

In this section we will introduce the required background for the novel methodology introduced in \cref{sec:methodology}.
In \cref{sec:bayescg} we present the salient details of BayesCG, \cref{sec:calib} discusses statistical calibration, and \cref{sec:krylov_prior} introduces the Krylov prior that our methodology modifies to achieve calibration.

\subsection{BayesCG} \label{sec:bayescg}

BayesCG is an iterative probabilistic linear solver for problems of the form $A x = b$, where $A \in \reals^{d\times d}$ is an invertible matrix and $b \in \reals^d$ is a vector, each given, while $x \in \reals^{d}$ is an unknown vector to be determined.
Most iterative solvers start from a user-specified \emph{initial iterate} $x_0$ and iteratively construct a sequence $(x_m)$, $x_m \in \reals^d, m \in \mathbb{N}$, such that $x_m \to x$ as $m \to \infty$.
BayesCG differs in that it takes a user-supplied Gaussian distribution as its initial iterate, $\mu_0 = \mathcal{N}(x_0, \Sigma_0)$, and returns a sequence $(\mu_m)$, $\mu_m =\mathcal{N}(x_m, \Sigma_m), m \in \mathbb{N}$, where $x_m \in \reals^d$ and $\Sigma_m \in \reals^{d \times d}_{\geq 0}$, the set of all positive semidefinite $d \times d$ matrices.

The means and covariances $x_m$ and $\Sigma_m$ are computed using Bayesian methods.
We first define an information operator
$$
\infoop_m(x) = S_m^\top A x
$$
where $S_m \in \reals^{d \times m}, S_m = [s_1, \dots, s_m]$ with $s_1, \dots, s_m \in \reals^d$.
The vectors $s_1, \dots, s_m$ are a set of linearly independent \emph{search directions}.
Linearity of the information operator results in a closed-form posterior distribution through the well-known Gaussian conditioning formula:
\begin{subequations}
\begin{align}
    x_m &= x_0 + \Sigma_0 A^\top S_m \Lambda_{m}^{-1} S_m^\top (b - A x_0) \\
    \Sigma_m &= \Sigma_0 - \Sigma_0 A^\top S_m \Lambda_{m}^{-1} S_m^\top A \Sigma_0 
\end{align} \label{eq:bayes_posterior}
\end{subequations}
where $\Lambda_{m} = S_m^\top A \Sigma_0 A^\top S_m$.

BayesCG arises from a particular choice of search directions given in \citet{Cockayne2019BayesCG}.

\begin{definition}[BayesCG Directions] \label{def:directions}
    The \emph{BayesCG directions} are given by $s_1 = r_0$ and
    $$
    s_m = r_{m-1} - (s_{m-1}^\top A \Sigma_0 A^\top r_{m-1}) \cdot \frac{s_{m-1}}{(s_{m-1}^\top A \Sigma_0 A^\top s_{m-1})},
    $$
    where $r_m = b - Ax_m$.
\end{definition}
These directions have the property of diagonalising $\Lambda_m$ \citep[see][Proposition~7]{Cockayne2019BayesCG}, resulting in an iterative expression for the posterior mean and covariance \citep[Proposition~6]{Cockayne2019BayesCG}.

In this paper we will restrict attention to the case of the inverse prior $\Sigma_0 = A^{-1}$, where $A$ is assumed to be symmetric and positive definite. 
With this prior the posterior from \cref{eq:bayes_posterior} simplifies to
\begin{subequations}
    \begin{align}
        x_m &= x_0 + S_m (S_m^\top A S_m)^{-1} S_m^\top (b - A x_0) \label{eq:cg_posterior_mean}\\
        \Sigma_m &= A^{-1} - S_m (S_m^\top A S_m)^{-1} S_m^\top, \label{eq:cg_posterior_cov} %
    \end{align} %
    \label{eq:inverse_prior_posterior} %
\end{subequations}
highlighting the dependence of the posterior covariance on $A^{-1}$ mentioned in \cref{sec:intro}.
The search directions from \cref{def:directions} also simplify to the following definition.

\begin{definition}[BayesCG directions under the inverse prior] \label{def:directions_invprior}
    The \emph{BayesCG directions} with $\Sigma_0 = A^{-1}$ are given by $s_1 = r_0$ and
    $$
    s_m = r_{m-1} - (s_{m-1}^\top A r_{m-1}) \cdot \frac{s_{m-1}}{(s_{m-1}^\top A s_{m-1})}
    $$
    where $r_m = b - Ax_m$.
\end{definition}

The search directions in \cref{def:directions_invprior} coincide with the search directions used in the CG algorithm (\cref{alg:cg}).
Moreover, when these directions are used, the posterior mean in \cref{eq:cg_posterior_mean} coincides with the CG iterate, and therefore enjoys all of the favourable properties of that iterate, such as its geometric rate of convergence\footnote{This is a worst-case result; typical rates can be much faster.}.
This highlights the rationale behind the inverse prior.
For other prior covariances analogous convergence results can be proven \citep[see][Proposition~10]{Cockayne2019BayesCG}, but the rate is typically slower.
Moreover, there is a growing literature on Gaussian process approximation \citep{Wenger2022CAGP,Wenger2024CAGPModelSelection,Stankewitz2024CAGP,Hegde2025GaussSeidel,Pfortner2025CAGP} that \emph{mandates} use of the $A^{-1}$ prior.

The search directions in \cref{def:directions_invprior} have several important theoretical properties. 
First recall that the Krylov space generated by a matrix $A$ and a vector $v$ is given by
$$
K_m(A, v) := \textup{span}(v, A v, \dots, A^{m-1} v).
$$
The \emph{grade} of a Krylov space is the smallest $M$ for which $K_{M+1}(B, v) = K_M(B, v)$.
We will also assume throughout that all Krylov spaces involved in this work have full grade $M = d$, but note that \cite{Reid2020BayesCG} carefully constructs BayesCG with $M < d$.
Defining the $m$\textsuperscript{th} \emph{residual} as $r_m = b - A x_m$, the required properties of the directions are given in the following theorem, which follows from \citet[Section~10.1 and Theorem~11.3.5]{Golub2013MatrixComputations}.

\begin{theorem} \label{thm:direction_properties}
    The CG directions have the following properties:
    \begin{enumerate}
        \item $s_i, s_j$ are $A$-orthogonal (i.e.\ $s_i^\top A s_j = 0$ if $i \neq j$).
        \item $\textup{span}(s_1, \dots, s_m) = K_m(A, r_0)$.
    \end{enumerate}
\end{theorem}

\begin{remark}
    Note that a more general version of \cref{thm:direction_properties} holds for arbitrary $\Sigma_0$; see \cite{Li2019Discussion}.
\end{remark}

As a final point, we note that the posterior distribution from BayesCG has a singular posterior covariance and, significantly for the discussion in the next section, that its null space depends on the true solution $x$.

\begin{theorem}[Propositions 1 and 4 in the rejoinder from \citet{Cockayne2019BayesCG}] \label{thm:posterior_support}
    The matrix $\Sigma_m$ from \cref{eq:cg_posterior_cov} has rank $d-m$ and its null space is given by $K_m(A, r_0)$.
\end{theorem}

To mitigate the dependence of \cref{eq:inverse_prior_posterior} on $A^{-1}$, a particular line of research advocates for employing the CG iterates to compute $x_m$ and calculating $\Sigma_m$ implicitly using some additional \emph{postiterations}.
This will be described in \cref{sec:krylov_prior}.
First however we must discuss \emph{calibration}.

\begin{algorithm}[t]
    \caption{Conjugate Gradient Method} \label{alg:cg}
    \begin{algorithmic}[1]
    \Procedure{\textsc{cg}}{$A$, $b$, $x_0$, $\epsilon$}
    \State $r_0 \gets b - A x_0$
    \State $s_1 \gets r_0$
    \State $k \gets 1$
    \While{\textbf{true}}
        \State $\alpha_k \gets r_{k-1}^\top r_{k-1} / (s_k^T A s_k)$
        \State $x_k \gets x_{k-1} + \alpha_k s_k$
        \State $r_k \gets r_{k-1} - \alpha_k A s_k$
        \If{$\norm{r_k} < \epsilon \norm{b}$}
            \State \textbf{break}
        \EndIf
        \State $\beta_k \gets \frac{r_k^\top r_k}{r_{k-1}^\top r_{k-1}}$
        \State $k \gets k+1$
        \State $s_k \gets r_{k-1} + \beta_{k-1} s_{k-1}$
    \EndWhile
    \State \Return $x_k$
    \EndProcedure
    \end{algorithmic}
\end{algorithm}

\subsection{Statistical Calibration} \label{sec:calib}

Formal calibration of probabilistic numerical methods has been increasingly studied in recent years, with \cite{Cockayne2022Calibrated,Cockayne2021Iterative} introducing definitions of strong calibration that were applied to (non-Bayesian) probabilistic iterative methods, while \cite{Hegde2025GaussSeidel} showed that calibration guarantees transfer from a probabilistic numerical method to a downstream application, in computation-aware GPs \citep{Wenger2022CAGP}.
We argue that this is particularly important for Krylov-based probabilistic linear solvers, where previously mentioned miscalibration has been observed repeatedly.

The central challenge of applying calibration definitions in this setting is the support of the posterior.
It is generally true that probabilistic numerical methods produce posteriors that are defined on a submanifold of the original space, and so any definition of calibratedness must hold in this setting.
The original definition of strong calibration \citep[Definition~8]{Cockayne2022Calibrated} is not suitable because of the required regularity assumptions \citep[Definitions~2 and~7]{Cockayne2022Calibrated} which essentially require that the posterior has full support.

\cite{Cockayne2021Iterative,Hegde2025GaussSeidel} circumvented this by exploiting the fact that the posteriors studied in those works, while not having \emph{full} support, nevertheless all had the \emph{same} support independent of $x$, allowing the more flexible definition from \citet[Definition~9]{Cockayne2021Iterative}; this essentially demands strong calibration on the common posterior support, and that the posterior mean matches the truth in the complement.

For Krylov-based posteriors we cannot apply a similar argument because, as highlighted in \cref{thm:posterior_support}, the support of the posterior depends on the Krylov space explored, which in turn depends on the true solution $x$.
This is randomised according to the prior in the definition of strong calibration. 
Applying the definition of weak calibration \cite[Definition 13]{Cockayne2022Calibrated} yields similar issues. 
We therefore turn to a yet weaker calibration condition which has been used for Krylov solvers in the past \citep{Cockayne2019BayesCG,Reid2023UQ}.

\begin{definition}[$\chi^2$-Calibration] \label{def:calibration}
    Consider a learning procedure on $\reals^d$, $\mu_m(x) = \mathcal{N}(x_m, \Sigma_m)$,
    where $\Sigma_m$ has rank $d-m$.
    Introduce the $Z$-statistic
    $$
    Z(x) = (x_m - x)^\top \Sigma_m^\dagger (x_m - x)
    $$
    where $\Sigma_m^\dagger$ is the Moore-Penrose pseudoinverse of $\Sigma_m$.
    We say that the learning producedure is \emph{$\chi^2$-calibrated} for the prior $\mu_0$ if it holds that
    $
    Z(X) \sim \chi^2_{d-m}
    $
    when $X$ is distributed according to $\mu_0$.
\end{definition}
While this is clearly weaker than the definitions of either strong or weak calibration, it solves the support problem through the projection onto $\reals$ through $Z$.
We will therefore use this condition throughout this paper.
We defer the important problem of a more general definition of strong calibration suitable for Krylov probabilistic linear solvers to future work.
\subsection{The Krylov Prior} \label{sec:krylov_prior}

To mitigate dependence of the posterior covariance on $A^{-1}$ described in \cref{sec:bayescg}, one approach that has been proposed is the \emph{Krylov prior}  \citep{Cockayne2019BayesCG,Reid2020BayesCG,Reid2023UQ}.

The Krylov prior can in the first instance be constructed by observing that if $\textup{span}\{ v_1, \dots, v_K \} \subseteq \reals^d$ , we can construct a Gaussian belief on $\reals^d$ as
\begin{align*}
    X &= x_0 + \sum_{i=1}^K v_i \phi_i^\frac{1}{2} \zeta_i \\
    &= x_0 + V_K \Phi^\frac{1}{2} Z
\end{align*}
where 
\begin{align*}
    Z &= [\zeta_1, \dots, \zeta_K]^\top \sim \mathcal{N}(0, I)\\
    \Phi &= \textup{diag}(\phi_1, \dots, \phi_K) \in \reals^{K \times K} \\
    V_{K} &= [v_1, \dots, v_K] \in \reals^{d \times K}.
\end{align*}
The law of $X$ is therefore $\mathcal{N}(x_0, V_K \Phi V_K^\top)$.

The idea behind the Krylov prior is to take $v_1, \dots, v_d$ to be $A$-orthonormal and such that $\textup{span}\{v_1, \dots, v_m\} = K_m(A, r_0)$.
Going forward, to simplify notation we will take
$$
V_m = S_m (S_m^\top A S_m)^{-\frac{1}{2}}
$$
for each $m > 0$.
We refer to $V_m$ as the \emph{normalised} search directions and $S_m$ as the \emph{unnormalised} search directions.
Note that in implementations we will use unnormalised directions, which will be discussed in \cref{sec:randomised_practical}.

We then have the following theoretical results concerning the posterior:
\begin{theorem}[Theorem 3.3 of \cite{Reid2020BayesCG}] \label{thm:reid_posterior}
    Let $\Sigma_0 = V_d \Psi V_d^\top$ for some (as-yet unspecified) diagonal $\Psi$.
    The posterior mean and covariance conditioned on information of the form
    $S_m^\top A x = S_m^\top b =: y_m$, $m < d$, satisfies
    \begin{align*}
        x_m &= x_0 + V_m V_m^\top r_0 \\
        &= x_0 + V_m \bm{\psi}_m \\
        \Sigma_m &= \perpdirections_m \perpscalings_m \perpdirections_m^\top
    \end{align*}
    where $\perpdirections_m = [v_{m+1} \dots v_d]$, $\perpscalings_m = \textup{diag}(\bar{\bm{\psi}}_m)$, $\bm{\psi}_m = [\psi_1, \dots, \psi_m]^\top$ and $\bar{\bm{\psi}}_m = [\psi_{m+1}, \dots, \psi_d]^\top$.
\end{theorem}

The next theorem proposes a choice of scalings $\psi_i$ under which a particular measure of the size of the posterior covariance matches the error precisely.

\begin{theorem}[Theorem 3.7 of \cite{Reid2020BayesCG}] \label{thm:reid_phi}
    Let
    $$
    \psi_i = \frac{r_{i-1}^\top r_{i-1}}{(s_i^\top A s_i)^{\frac{1}{2}}},
    $$
    $i=1,\dots,d$.
    Then $\textup{tr}(A \Sigma_m) = \norm{x - x_m}_A^2$.
\end{theorem}

The rationale for the interest in $\textup{trace}(A \Sigma_m)$ is due to \citet[Lemma~3.5]{Reid2020BayesCG}, which states that $\mathbb{E}\| X - x_m \|^2_A = \textup{trace}(A\Sigma_m)$ when $X \sim N(x_m, \Sigma_m)$. 
Combining this with the result above, the intuition is that the expected $A$-norm distance between samples from $X$ and its mean matches the CG error with this particular choice of weights.

\begin{remark}
    Note that \cite{Reid2020BayesCG} represents these computations using $\phi_i = \psi_i^2$.
\end{remark}

\begin{remark}
    Under this choice of $\psi_i$ the prior and posterior covariance are dependent on $b$ (and thus $x$) and so this could be regarded as an empirical Bayesian approach.
    This does not preclude application of the calibration ideas defined in \cref{sec:calib}, since the learning procedure is only required to return a measure on $\reals^d$; it is permissible for this measure to have a complex dependance on $x$.
\end{remark}

From \citet[Theorem SM3.2]{Reid2020BayesCG}, we have that $\psi_i$ can be equivalently calculated as
$$
\psi_i = v_i^\top r_0 .
$$
It will also be helpful to establish the identity:
\begin{equation}
\frac{\psi_i}{(s_i^\top A s_i)^{\frac{1}{2}}} = \alpha_i
\end{equation}
where $\alpha_i$ is as given in \cref{alg:cg}.

While \cref{thm:reid_phi} shows that the width of the posterior covariance is ``about right'', the next theorem shows that with this choice of $\phi_i$, the posterior is decidedly not calibrated.

\begin{theorem} \label{thm:reid_bad_uq}
    With the choice from \cref{thm:reid_phi}, the posterior is not $\chi^2$-calibrated.
\end{theorem}

\begin{proof}
    From \citet[Theorem~4.13]{Reid2023UQ} $Z(x) = d-m$, independent of $x$.
    Therefore $Z(X)$ is a Dirac mass at $d-m$, and so the posterior is not $\chi^2$-calibrated.
\end{proof}

The proof given above highlights that the miscalibration is due to a \emph{perfect cancellation} of the randomness induced by randomising $x$ according to the prior, caused by the choice of scales in \cref{thm:reid_phi}.

\subsubsection{Practical Implementation} \label{sec:reid_practical}

It must be mentioned that while the above is theoretically appealing, it is not practical.
This is owing to the requirement that the full Krylov basis $V_d$ be computed in order to define the posterior.
The process of computing this basis is identical to that of running CG to convergence, which would make uncertainty quantification (UQ) irrelevant.

To address this, \cite{Reid2023UQ} propose to instead run CG for $t$ iterations, where $m < t \ll d$, and use the $t-m$ \emph{postiterations} to construct the posterior.
Due to the fast convergence of CG, this means that the most dominant subspace of $\textup{span}(\perpdirections_m)$ will still be captured with only a small number of postiterations.
On the other hand, the resulting posterior will assign zero mass to the space $\textup{span}(\perpdirections_t)$, falsely implying that there is no error in that space, which could be problematic for some applications.
\citet{Cockayne2019BayesCG} proposed an approach to address this by computing an orthonormal basis of $\perpdirections_t$ using a QR decomposition, but this did not seem to work well empirically.
We will explore the impact of this in \cref{sec:experiments}, but will leave attempts to correct it for future work.

A further objection to this idea is that, if one has the budget to perform $t$ iterations, why not use the better estimate of $x$ provided by these iterations with no UQ? We argue that in some applications and downstream tasks, having a better calibrated probabilistic linear solver is more useful than a more accurate (but still high error) estimate, a conclusion supported by results in \cite{poot2025bayesianfiniteelementmethod}. 
This idea is explored empirically in \cref{sec:inverse}. 

In the next section we propose our novel, randomised version of the Krylov prior that has more favourable calibration properties.

\section{Randomised Postiterations} \label{sec:methodology}

\subsection{A Randomised Krylov Prior}

First recall the following results from the proof of \cite[Theorem~4.13]{Reid2023UQ}:
\begin{align*}
    x &= x_0 + V_d \bm{\psi}_d \\
    x_m &= x_0 + V_m \bm{\psi}_m \\
    \Sigma_m &= \perpdirections_m \perpscalings_m \perpdirections_m^\top.
\end{align*}
An immediate consequence of this is that
\begin{align*}
x - x_m &= \perpdirections_m \bar{\bm{\psi}}_m.
\end{align*}
This leads us to consider \emph{randomising} $x_m$ to reintroduce the randomness that, as previously mentioned, is perfectly cancelled by the choice in \cref{thm:reid_phi}.

\begin{definition}[Randomised Postiterations] \label{def:randomised_posterior}
    Let $x_m$ be the posterior mean in \cref{eq:cg_posterior_mean}.
    The randomised posterior mean is given by $\tilde{x}_m = x_m + \perpdirections_m e_m$
    where $\perpdirections_m$ is as given in \cref{thm:reid_posterior} and
    $$e_m \sim \mathcal{N}(\bar{\bm{\psi}}_m, \perpscalings_m),$$
    where $\bar{\bm{\psi}}_m$ and $\perpscalings_m$ are as given in \cref{thm:reid_posterior,thm:reid_phi}.
\end{definition}

Our first result below provides a (pointwise) guarantee that the randomised posterior follows a distribution that is intuitively correct.
Notably this is stronger than required for \cref{def:calibration}.
\begin{theorem} \label{thm:perturbation}
    The randomised posterior from \cref{def:randomised_posterior} satisfies $L_m^\dagger (x - \tilde{x}_m) \sim \mathcal{N}(0, I_{d-m})$,
    where $L_m$ is an arbitrary left square-root of $\Sigma_m$.
\end{theorem}

\begin{proof}
    First note that a left square-root of $\Sigma_m$ is given by
    $$
    L_m = \perpdirections_m\perpscalings_m^\frac{1}{2} .
    $$
    We can straightforwardly calculate the Moore-Penrose pseudoinverse of this as
    $$
    L_m^\dagger = \perpscalings_m^{-\frac{1}{2}} (\perpdirections_m^\top \perpdirections_m)^{-1} \perpdirections_m^\top.
    $$
    \cite[Lemma 4.12]{Reid2023UQ}. 
    We then have that
    \begin{align*}
        L_m^\dagger (x - \tilde{x}_m) &= L_m^\dagger(x - x_m) - L_m^\dagger \perpdirections_m e_m \\
        &= \perpscalings_m^{-\frac{1}{2}} (\perpdirections_m^\top \perpdirections_m)^{-1} \perpdirections_m^\top \perpdirections_m (\bar{\bm{\psi}}_m - e_m) \\
        &= \perpscalings_m^{-\frac{1}{2}}  (\bar{\bm{\psi}}_m - e_m)
    \end{align*}
    which clearly has the distribution stated in the theorem.
\end{proof}

\cref{thm:perturbation} has a notably different interpretation than other calibration results in the literature because the distribution is not induced by randomness in $x$ but by randomness in $\tilde{x}_m$.
Nevertheless we emphasise that because both $\perpdirections_m$ and $\perpscalings_m$ are functions of $x$, the posterior depends highly nontrivially on the true solution which complicates calibration analysis.
The next result establishes $\chi^2$-calibratedness of the randomised posterior.

\begin{corollary} \label{corr:calibration}
    The randomised posterior from \cref{def:randomised_posterior} is $\chi^2$-calibrated.
\end{corollary}
\begin{proof}
    Since $\Sigma_m = L_m L_m^\top$, it follows that $\Sigma_m^\dagger = (L_m^\top)^\dagger L_m^\dagger$.
    We therefore have that
    $$
    Z(x) = \norm{L_m^\dagger (x - \tilde{x}_m)}_2^2.
    $$
    From \cref{thm:perturbation}, this is the squared norm of a $\mathcal{N}(0, I_{d-m})$ distributed random variable, which completes the proof.
\end{proof}

An interesting property of \cref{thm:perturbation} and \cref{corr:calibration} is that both results are independent of the choice of $\perpscalings_i$; because $e_m$ has this as its covariance, the correct distribution will always be obtained.
This reflects a general point of calibration results, that calibration is a property that is independent of other favourable properties of the posterior, such as rate of concentration, or indeed, whether the posterior concentrates at all (e.g.\ provided the covariance and the mean error diverge commensurately, the posterior can still be calibrated).
To ensure that the posterior reflects the true error we recommend setting $\perpscalings_i$ in the same way as  in \cref{thm:reid_phi}.
Since these quantities are already computed during the postiterations, this does not change the cost of computation.

\begin{algorithm}[t]
    \begin{algorithmic}[1]
    \Procedure{\textsc{cg\_rpi}}{$A$, $b$, $x_0$, $\epsilon_1$, $\epsilon_2$}
    \State $r_0 \gets b - Ax_0$
    \State $s_1 \gets r_0$
    \State $k \gets 1$
    \State $\tilde{x}_k \gets \textsc{nothing}$
    \State $L_k \gets [\;] \in \reals^{d \times 0}$
    \While{\textbf{true}}
        \State $\alpha_k \gets \frac{r_{k-1}^\top r_{k-1}}{s_k^\top A s_k}$
        \State $x_k \gets x_{k-1} + \alpha_k s_k$
        \State $r_k \gets r_{k-1} - \alpha_k A s_k$
        \State $\beta_k \gets \frac{r_k^\top r_k}{r_{k-1}^\top r_{k-1}}$
        \If{$\neg\textsc{isnothing}(\tilde{x}_k)$}
            \If{$\|r_k\| > \epsilon_2 \norm{b}$}
                \State $z_k \sim \mathcal{N}(0, 1)$
                \State $\tilde{x}_k \gets \tilde{x}_{k-1} + \alpha_k (z_k + 1) s_k$
                \State $L_k \gets \begin{bmatrix} L_{k-1} & \alpha_k s_k\end{bmatrix}$
            \Else
                \State \textbf{break}
            \EndIf
        \ElsIf{$\|r_k\| \leq \epsilon_1 \norm{b}$}
            \State $\tilde{x}_k \gets x_k$
        \EndIf
        \State $k \gets k+1$
        \State $s_k \gets r_{k-1} + \beta_{k-1} s_{k-1}$
    \EndWhile
    \State \textbf{return} $\tilde{x}_k, L_k $
    \EndProcedure
    \end{algorithmic}
\caption{The Randomised Postiteration procedure described in \cref{sec:methodology}. } \label{alg:randomised_postiteration}
\end{algorithm}

\subsubsection{Practical Implementation} \label{sec:randomised_practical}

We now consider practical implementation of the procedure described in \cref{thm:perturbation}.
Two main optimisations are discussed below, with the resulting algorithm given in \cref{alg:randomised_postiteration}.

\paragraph{Recursion with Unnormalised Directions}
For the first $m$ iterations the algorithm is identical to \cref{alg:cg}, so we focus on the differences in the postiteration period where $k > m$.
Note that $\tilde{x}_k$ can be computed iteratively by setting $\tilde{x}_m = x_m$, and $\tilde{x}_k = \tilde{x}_{k-1} + v_k e_k$, where $e_k \sim \mathcal{N}(-\psi_k, \psi_k^2)$.
Next note that typically computations are performed using the directions $s_k$ rather than $v_k$ for stability.
The randomised posterior mean can therefore be computed as
\begin{align*}
\tilde{x}_k &= \tilde{x}_{k-1} + \left( \frac{\psi_k}{(s_k^\top A s_k)^{\frac{1}{2}}} + \frac{\psi_k}{(s_k^\top A s_k)^{\frac{1}{2}}} z_k \right) s_k  \\
&= \tilde{x}_{k-1} + \alpha_k\left(z_k + 1 \right) s_k
\end{align*}
with $z_k \sim \mathcal{N}(0, 1)$.

Note that it is also straightforward to calculate the posterior covariance recursively.
As in \cite{Cockayne2019BayesCG} we opt to represent this using the low-rank factors $L_k$, where $L_m = [\;] \in \reals^{d \times 0}$ and
$$
L_k = \begin{bmatrix} L_{k-1} & \alpha_k s_k \end{bmatrix}
$$
for $k > m$.
Then we have that $\Sigma_k = L_k L_k^\top$.

\paragraph{Truncated Postiterations}
As was discussed in \cref{sec:reid_practical} it is not possible to realise $\bar{V}_m$ in practice.
We propose to adopt the same approach discussed there of truncating the postiterations.
That is, we perform $t$ total iterations of CG, $m < t \ll d$, and use the $t-m$ postiterations to calibrate the posterior.
To implement this truncation we introduce a second tolerance, $\epsilon_2 < \epsilon_1$ in \cref{alg:randomised_postiteration}.
The parameter $\epsilon_1$ controls at which point standard CG iterations are terminated as with $\epsilon$ in \cref{alg:cg}, while $\epsilon_2$ controls how many postiterations are performed.
That is, $\epsilon_1$ defines $m$ while $\epsilon_2$ defines $t$. Standard CG is then recovered in the case where $\epsilon_1 = \epsilon_2$.
To avoid computing too many further postiterations we suggest setting $\epsilon_2 = \delta \epsilon_1$ for some $0 < \delta < 1$ (e.g.\ $\delta = 0.1$).
The impact of the choice of $\delta$ will be explored in \cref{sec:experiments}.

\section{Experiments} \label{sec:experiments}

\subsection{Simulation-Based Calibration}

Since \cref{corr:calibration} provides fairly weak calibration guarantees, we begin by examining posterior calibratedness empirically.
To do this we employ the simulation-based calibration (SBC) test of \cite{Talts2018SBC}, described in \cref{alg:sbc}.
We use a Gaussian prior $\mu_0 = \mathcal{N}(0, A^{-1})$, and set $A = W D W^\top$, where $W$ is an orthonormal matrix drawn from the Haar measure on such matrices, while $D$ is diagonal with strictly positive entries drawn independently from Exp(1). All SBC experiments use a symmetric positive definite matrix $A \in \reals^{d\times d}$ of dimension $d=100$. This matrix is kept fixed across the simulations for individual runs of the SBC test (\cref{fig:sbc_hist_e-2,fig:sbc_hist_e-5,fig:sbc_hist_TR}) but redrawn for repeated runs (\cref{fig:ks_eps2_plot}).
If calibrated, the histogram of implied posterior CDF evaluations should be distributed as $\mathcal{U}(0,1)$.

\begin{algorithm}[t!]
    \caption{Simulation-Based Calibration. $\Phi$ denotes the CDF of the standard Gaussian distribution.} \label{alg:sbc}
    \begin{algorithmic}[1]
    \Require Prior $\mu_0 = \mathcal{N}(u_0, A^{-1})$, number of simulations $N_\textsc{sim}$, test vector $w$.
    \For{$i = 1$ to $N_\textsc{sim}$}
        \State $x^{(i)} \sim \mu_0$
        \State $b^{(i)} \gets A x^{(i)}$
        \State $x_m^{(i)}, L_m^{(i)} = \textsc{rpi}(A, b^{(i)}, x_0, \epsilon_1, \epsilon_2)$
        \State $t_i = \Phi\left(\frac{w^\top (x_m^{(i)} - x^{(i)})}{(w^\top L_m^{(i)} (L_m^{(i)})^\top w)^\frac{1}{2}}\right)$
    \EndFor
    \State Plot histogram of $\{t_i\}_{i=1}^{N_\textsc{sim}}$ and compare to $\mathcal{U}(0,1)$.
    \end{algorithmic}
\end{algorithm}

\subsubsection{Results}

The results of the simulation-based calibration tests demonstrate that the posterior is indeed calibrated for sufficiently small choices of $\epsilon_2$. 

\cref{fig:sbc_hist_TR} presents the results from $10,000$ simulations of the procedure outlined in \cref{alg:sbc}, applied to the postiterations approach of \citet{Reid2023UQ}, with $\epsilon_2 = 10^{-5}$ and $\epsilon_1=10^{-1}$.  As anticipated, the histogram exhibits a $\cup$-shaped pattern, indicating poor calibration with the computed posteriors being, on average, under-dispersed relative to the true posterior \citep{Talts2018SBC}. In comparison, the results from the randomised version of the postiteration algorithm, with the same number of simulations and choice of the thresholds appear uniform in \cref{fig:sbc_hist_e-5}, suggesting that the posterior is well-calibrated for this choice of thresholds.

However, the randomised algorithm does exhibit under-dispersed posteriors for larger choices of $\epsilon_2$. This is seen in \cref{fig:sbc_hist_e-2} where $\epsilon_2 = 10^{-2}$ and $\epsilon_1=10^{-1}$. We again observe the characteristic $\cup$-shaped histogram indicating generally over-confident posteriors but note that it is less extreme than the results from \cref{fig:sbc_hist_TR}, even with the relatively larger $\epsilon_2$.

A quantitative comparison of the two algorithms is provided in \cref{fig:ks_eps2_plot}. The Kolmogorov-Smirnov (KS) test statistic is used to estimate the difference between the SBC results and the standard uniform distribution, for different choices of thresholds. Results were obtained by averaging the KS statistics from 1000 SBC runs, each using 1000 simulations. The median KS statistics along with their interquartile range are plotted. The grey line indicates the average KS statistic between samples from the standard uniform and the standard uniform itself. While the deterministic postiteration algorithm maintains the same level of poor calibration across the varying thresholds, the randomised algorithm's SBC results converges to being uniformly distributed and therefore calibrated.

\begin{figure*}
\begin{subfigure}{0.5\textwidth}
    \includegraphics{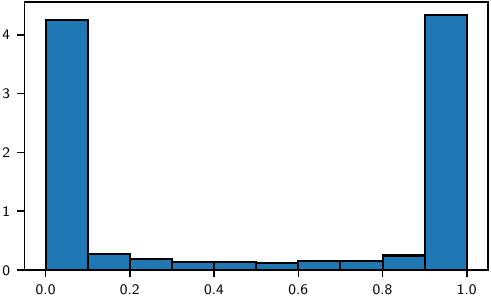}
    \caption{\textsc{pi}; $\epsilon_1=10^{-1}$, $\epsilon_2 = 10^{-5}$.} \label{fig:sbc_hist_TR}
\end{subfigure}
\begin{subfigure}{0.5\textwidth}
    \includegraphics{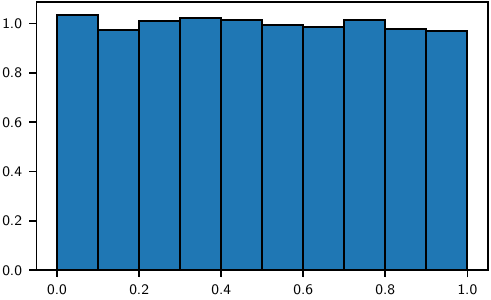}
    \caption{\textsc{rpi}; $\epsilon_1=10^{-1}$, $\epsilon_2 = 10^{-5}$.} \label{fig:sbc_hist_e-5}
\end{subfigure}

\begin{subfigure}{0.5\textwidth}
    \includegraphics{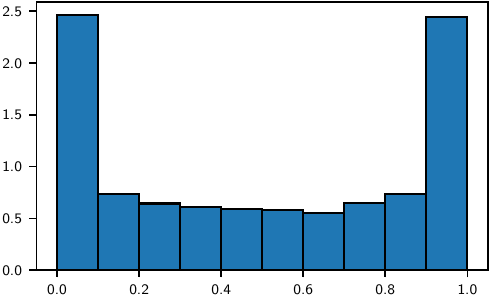}
    \caption{\textsc{rpi}; $\epsilon_1=10^{-1}$, $\epsilon_2 = 10^{-2}$.} \label{fig:sbc_hist_e-2}
\end{subfigure}
\begin{subfigure}{0.5\textwidth}
    \includegraphics{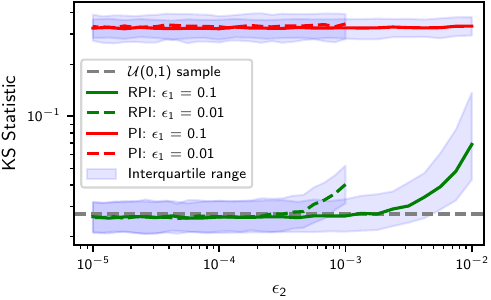}
    \caption{Kolmogorov-Smirnov test statistics for varying $\epsilon_1$ and $\epsilon_2$. } \label{fig:ks_eps2_plot}
\end{subfigure}
\caption{Simulation-based calibration test results. \cref{fig:sbc_hist_TR,fig:sbc_hist_e-5,fig:sbc_hist_e-2} give results for different values of $\epsilon_1, \epsilon_2$ while \cref{fig:ks_eps2_plot} gives KS test statistics as a function of $\epsilon_2$, for multiple replicates of the experiment.}
\end{figure*}

\begin{remark}[Calibration, for what measure?]
    While \cref{alg:sbc} tests for calibration with-respect-to $\mathcal{N}(x_0, A^{-1})$, in fact \cref{thm:perturbation} and \cref{corr:calibration} show that we should achieve calibration \emph{independent} of the prior.
    This is because \cref{alg:randomised_postiteration} is independent of the prior.
    We have opted to use $A^{-1}$ in \cref{alg:sbc} only because this is the prior covariance under which BayesCG coincides with CG.
\end{remark}

\subsection{Inverse Problem} \label{sec:inverse}

We now consider incorporating the randomised postiteration procedure into an inverse problem.
This has been considered several times in the literature, including for linear solvers \citep[see][]{Cockayne2019BayesCG}.
We consider a linear system arising from the discretisation of a partial differential equation describing steady-state flow through a porous medium, given as:
\begin{align*}
    -\nabla\cdot(k(\spatialvar; \theta) \nabla u(\spatialvar; \theta)) &= f(\spatialvar) & \spatialvar&\in D\\
    u(\spatialvar; \theta) &= b(\spatialvar) & \spatialvar &\in \partial D_1 \\
    \nabla u(\spatialvar; \theta) \cdot n &= 0 & \spatialvar &\in \partial D_2
\end{align*}
where $n$ is the outward pointing normal vector.
For the domains we take
\begin{subequations}
\begin{align}
    D &= (0,1)^2 \\
    D_1 &= \bigcup_{\spatialvar_1\in [0,1]} \set{(\spatialvar_1,0), (\spatialvar_1,1)} \\
    D_2 &= \bigcup_{\spatialvar_2\in [0,1]} \set{(0,\spatialvar_2), (1,\spatialvar_2)}.
\end{align} \label{eq:inverse_pde}
\end{subequations}
In other words, $D_1$ is the top and bottom edges of the unit square and $D_2$ is the left and right edges.
For the forcing we take $f(\spatialvar) = 0$ and for the boundary term we take $b(\spatialvar) = (1-\spatialvar_1)(1-\spatialvar_2) + \spatialvar_1 \spatialvar_2$.

The inverse porosity $k(\spatialvar; \theta)$ is taken to be
$$
k(\spatialvar; \theta) = 1 +  \mathbb{1}_{D_\theta}(\spatialvar) \cdot \exp(\theta)
$$
where $\mathbb{1}_{D_\theta}(\spatialvar)=1$ if $\spatialvar \in D_\theta$ and $0$ elsewhere.
We take $D_\theta = [0.25, 0.75]^2$.
Thus this problem represents steady-state corner to corner flow through a domain with a low porosity central region that impedes the flow.

We set this up as a Bayesian inference problem to determine $\theta$.
To accomplish this we use the data-generating model
$$
y_i = f(\spatialvar_i) + \zeta_i, i=1,\dots,N.
$$
We take $N=4$ and $\spatialvar_1, \dots, \spatialvar_4$ are taken to be the four corners of the region $D_\theta$, while $\zeta_i \stackrel{\text{iid}}{\sim} \mathcal{N}(0, \sigma^2)$, $\sigma = 0.01$.
Data was generated using the data-generating parameter $\theta^\dagger = 2$.
Letting $\bm{z} = [z_1,\dots,z_N]$ this implies a Gaussian likelihood
\begin{equation} \label{eq:likelihood}
p(y \mid \theta) = \mathcal{N}(y; u(\bm{\spatialvar}; \theta), \sigma^2 I).
\end{equation}
We endow $\theta$ with a Gaussian prior $\theta \sim \mathcal{N}(0,1)$.
Note however that the posterior $p(\theta \mid y)$ is non-Gaussian due to the nonlinear dependence of $u$ on $\theta$.
We will therefore sample the posterior using Markov-chain Monte-Carlo (MCMC) methods.

To turn this into a linear system we discretise the PDE in \cref{eq:inverse_pde} using the finite-element method, as implemented in the Julia package \texttt{Gridap} \citep{Badia2020,Verdugo2022}.
We omit most of the details here, but the discretisation results in a (sparse) linear system $A_\theta x_\theta = b$ that must be solved for each $\theta$ to evaluate the likelihood, where $b$ is independent of $\theta$ and $x_\theta$ has components that approximate the solution $u(\spatialvar; \theta)$ at the nodal points on a fine mesh over the domain.
The points in $\bm{\spatialvar}$ are included in the nodal point set, and so the value of $u(\bm{\spatialvar}; \theta)$ may be related to $x_\theta$ as
$$
u(\bm{\spatialvar}; \theta) \approx w x_\theta
$$
for some $w \in \reals^{N \times d}$, whose entries are either $0$ or $1$, to ``pick out'' the degrees of freedom in $x_\theta$ associated with $\bm{z}$.
The likelihood in \cref{eq:likelihood} is thus approximated as
$$
p(y \mid \theta) \approx \hat{p}(y \mid \theta) := \mathcal{N}(y; w x_\theta, \sigma^2 I).
$$
The grid size for this problem was such that the dimension of the systems involved is $d = 475$. 
Naturally this approach introduces discretisation error due to the use of an approximate PDE solver, but we will assume for this work that the system is discretised finely enough that the error is negligible.

We consider three approaches to solve the linear system:
\begin{enumerate}[leftmargin=4em]
    \item[\textsc{exact}:] Explicit solution using the sparse Cholesky routine \texttt{CHOLMOD} \citep{Chen2008} as is default in Julia for sparse Hermitian matrices.
    \item[\textsc{cg}:] CG with fixed tolerance $\epsilon$.
    \item[\textsc{pi}:] Postiterations in the style of \citet{Reid2023UQ}, with tolerances $\epsilon_2 = \epsilon$ and $\epsilon_1 = 10 \epsilon_2$.
    \item[\textsc{rpi}:] Randomised postiterations with the same tolerances as \textsc{pi}.
\end{enumerate}
The comparison between \textsc{cg} and \textsc{rpi} is intended to demonstrate that, in a constrained computing environment, it is sometimes better to use some of the computational budget to perform postiterations and calibrate the posterior, rather than spending all iterations on CG.
When using postiterations we modify the likelihood from \cref{eq:likelihood} to incorporate uncertainty from the probabilistic solver as has been described several times in the literature (e.g.\ \cite{Cockayne2019BayesCG}, which described this approach for linear systems).
The resulting likelihood for \textsc{rpi} is given by
$$
p_\textsc{pn}(y \mid \theta) = \mathcal{N}(y; w \tilde{x}_m(\theta), \sigma^2 I + w \Sigma_m(\theta) w^\top)
$$
while for \textsc{pi} we replace $\tilde{x}_m$ with $x_m$.
Thus, this likelihood incorporates the uncertainty due to postiterations by inflating the likelihood variance with the factor $w \Sigma_m(\theta) w^\top$.

For this simple inverse problem we applied the random walk Metropolis algorithm \citep{Roberts2004}.
We used Gaussian proposals $\theta_n \mid \theta_{n-1} \sim \mathcal{N}(\theta_{n-1}, \eta)$ with $\eta = 0.2$ for \textsc{exact} and \textsc{cg}, and $0.4$ for \textsc{pi} and \textsc{rpi} (which have wider posteriors due to the variance inflation).
We performed $N = 10,000$ iterations for each chain.
In \cref{fig:inverse} kernel density estimates of the three posteriors are displayed with $\epsilon=0.1$, to ensure that the error is still larger than the noise level.
As can be seen, the MAP point of \textsc{exact} is close to $\theta^\dagger$, while the other methods show some bias: \textsc{cg} and \textsc{pi} positive and negative bias, while for \textsc{rpi} the bias is negative but much smaller.
Notably, however, the posterior for \textsc{rpi} is significantly wider than for \textsc{cg} or \textsc{pi}, i.e.\ the posterior better represents the uncertainty due to computation, and places more mass around $\theta^\dagger$.
Furthermore, both \textsc{rpi} and \textsc{pi} are also much closer to the prior, again reflecting increased scepticism about the data due to reduced computation.
\begin{figure}
    \includegraphics{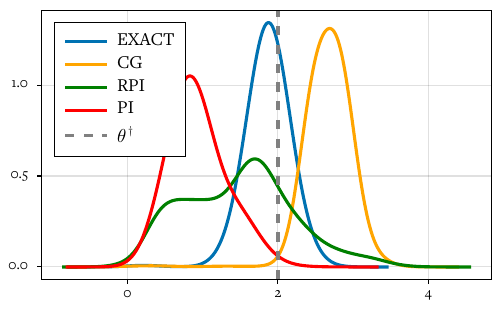}
    \caption{Posteriors for the inverse problem described in \cref{sec:inverse}. } \label{fig:inverse}
\end{figure}

\section{Conclusion} \label{sec:conclusion}

In this paper we have introduced a randomised version of the postiteration algorithm from \cite{Reid2023UQ} that has improved calibration properties.
These properties have been examined empirically, and we have performed a more extensive computational examination.
We have also shown how using some of a constrained computational budget for calibration can give better results in uncertainty propagation problems, such as inverse problems, than using all of the budget to compute a low-accuracy estimator.

There are several avenues for future work.
Firstly, the proposed procedure still suffers from the deficiency of \cite{Reid2023UQ}, that (when truncated) it assigns zero posterior mass to a subspace in which there is still error.
It would be useful to examine whether techniques such as those described in \cite{Cockayne2019BayesCG} can be extended to correct this, by assigning mass to this space without incurring the overhead of computing a basis for that space.
Secondly, the calibration results we have derived are not as strong as we hoped.
This is largely due to issues with the definition of strong calibration, which make it challenging to establish results for singular posteriors which arise in probabilistic linear solvers.
A more flexible definition of strong calibration would be needed to address this, but this is beyond the scope of this paper.

A second line of work is in applications.
Recent papers \citep{Poot2024} have proposed strategies for calibrating the posterior in Bayesian finite element solvers that are similar to the postiteration ideas in \cite{Reid2023UQ}.
It would be interesting to explore whether the randomisation strategies discussed in this paper can be applied to that infinite-dimensional setting to yield similar calibration improvements.

\section*{Acknowledgements}
NV was supported by the UKRI AI Centre for Doctoral Training in AI for Sustainability [EP/Y030702/1].
JC was supported by EPSRC grant EP/\allowbreak{}FY001028/F1.

\bibliography{citations.bib}
\end{document}